\documentclass[twocolumn,twoside] {IEEEtran}

\usepackage{spconf,epsfig,graphics,subfigure}
\usepackage{amsmath,amssymb,amsthm,multirow,balance,dsfont,hyperref}
\usepackage{cite,cleveref}
\usepackage[table]{xcolor}
\usepackage{array}
\usepackage[ruled,vlined]{algorithm2e}

\newtheorem{lemma}{Lemma}

\usepackage{color}

\newcommand{\bd}{\boldsymbol{d}}

\newcommand{\bu}{\boldsymbol{u}}
\newcommand{\bv}{\boldsymbol{v}}
\newcommand{\bw}{\boldsymbol{w}}

\newcommand{\bx}{\boldsymbol{x}}
\newcommand{\by}{\boldsymbol{y}}

\newcommand{\bpsi}{\boldsymbol{\psi}}

\newcommand{\bzeta}{\boldsymbol{\zeta}}

\newcommand{\cA}{\mathcal{A}}

\newcommand{\cC}{\mathcal{C}}
\newcommand{\cD}{\mathcal{D}}
\newcommand{\cE}{\mathcal{E}}

\newcommand{\cR}{\mathcal{R}}

\newcommand{\cJ}{\mathcal{J}}

\newcommand{\cN}{\mathcal{N}}
\newcommand{\cP}{\mathcal{P}}

\newcommand{\cV}{\mathcal{V}}

\newcommand{\cW}{\mathcal{W}}
\newcommand{\cZ}{\mathcal{Z}}
\newcommand{\cw}{{\scriptstyle\mathcal{W}}}

\newcommand{\ssd}{{\scriptscriptstyle{D}}}
\newcommand{\ssw}{{\scriptscriptstyle{W}}}
\newcommand{\ssu}{{\scriptscriptstyle{U}}}

\newcommand{\ssz}{{\scriptscriptstyle{Z}}}

\newcommand{\cx}{{\scriptstyle\mathcal{X}}}
\newcommand{\cy}{{\scriptstyle\mathcal{Y}}}

\newcommand{\cz}{{\scriptstyle\mathcal{Z}}}
\newcommand{\ccw}{{\scriptscriptstyle\mathcal{W}}}
\newcommand{\ccz}{{\scriptscriptstyle\mathcal{Z}}}
\newcommand{\ccd}{{\scriptscriptstyle\mathcal{D}}}

\newcommand{\ccx}{{\scriptscriptstyle\mathcal{X}}}

\newcommand{\ccy}{{\scriptscriptstyle\mathcal{Y}}}

\newcommand{\bcw}{\boldsymbol{\cw}}

\newcommand{\bcy}{\boldsymbol{\cy}}

\newcommand{\bxt}{\widetilde \bx}
\newcommand{\wt}{\widetilde w}

\newcommand{\expec}{\mathbb{E}}

\newcommand{\col}{\text{col}}

\newcommand{\diag}{\text{diag}}

\DeclareMathOperator*{\argmin}{argmin}
\DeclareMathOperator*{\minimize}{minimize}
\DeclareMathOperator*{\st}{subject~to}

\newcolumntype{C}[1]{>{\centering\arraybackslash}m{#1}}

\begin{document}
\title{{Decentralized} learning in the presence of low-rank noise\vspace{-0.2cm} }
\name{Roula Nassif$\,^{(1)}$, Virginia Bordignon$\,^{(2)}$,  Stefan Vlaski$\,^{(3)}$, Ali H. Sayed$\,^{(2)}$
\vspace{-0.5cm} }

\address{
\small $\,^{(1)}$Universit\'e C\^ote d'Azur, France\\
\small $\,^{(2)}$Ecole Polytechnique F\'ed\'erale de Lausanne, Switzerland\\
\small $\,^{(3)}$Imperial College London, UK \\
}

\ninept

%


\maketitle

\begin{abstract}
Observations collected by agents in a network may be unreliable due to observation noise or interference. This paper proposes a distributed algorithm that allows each node to improve the reliability of its own observation by relying solely on local computations and interactions with immediate neighbors, assuming that the field (graph signal) monitored by the network lies in a low-dimensional subspace and that a low-rank noise is present in addition to the usual full-rank noise. While oblique projections can be used to project measurements onto a low-rank subspace along a direction that is oblique to the subspace, the resulting solution is not distributed. Starting from the centralized solution, we propose an algorithm  that performs the oblique projection of the overall set of observations onto the signal subspace in an iterative and distributed manner. We then show how the oblique projection framework can be extended to handle distributed learning and adaptation problems over networks.

\end{abstract}

\vspace{-0.1cm}
\section{Introduction}
We consider $N$ agents in a network, collecting data with the objective of collaboratively solving some inference task. The locally observed data may be unreliable due to the presence of measurement noise or interference. Most prior literature  treats noise as a full-rank process in the measurement space. The work~\cite{barbarossa2009distributed}, for instance, models the desired graph signal as a vector that lies in a low-rank subspace and the noise as a vector that may fall anywhere in the observation space. Orthogonal projection techniques have been used  to recover the original signal and to mitigate the effect of noise. The projection in~\cite{barbarossa2009distributed} is carried out through a distributed network, with no fusion center, where each node exchanges information only with its neighbors. While centralized solutions can be powerful, decentralized solutions are more attractive since they are more robust, and allow agents to keep their local data private~\cite{sayed2014adaptation}. Distributed algorithms and their ability to perform  globally optimal processing tasks (such as minimizing aggregate {sums} of individual costs, solving constrained optimization problems, etc.) have been widely studied in the literature~\cite{sayed2014adaptation,sayed2013diffusion, 
nassif2020multitask,nassif2020adaptation,nassif2020adaptation2,dilorenzo2020distributed,dimakis2010gossip,nedic2009distributed,braca2008enforcing,chouvardas2011adaptive,
mota2015distributed,koppel2016proximity,xiao2004fast}. 

In this paper, we consider distributed estimation  in the presence of low-rank, or structured, noise in addition to the usual full-rank noise.   In the first part, and for motivational purposes, we consider a de-noising problem where the graph signal to be estimated  lies in a low-dimensional subspace. That is, we consider a connected network (or graph) of $N$ nodes and we let ${y}_{k}$ denote the scalar measurement collected by node $k$. Let $y=\col\{y_1,\ldots,y_N\}$ denote the collection of observations  from across the network. We assume the following  model:
 \vspace{-1mm}
\begin{equation}
\label{eq: linear data model}
y=\underbrace{Wx_{\ssw}}_{\text{useful signal}}+\underbrace{{Z}{x_{\ssz}}+{v}}_{\text{noise component}},
\end{equation}
where $W$ is an $N\times P$ full-column rank matrix with $P\ll N$ and $x_{\ssw}$ is a $P\times 1$ column vector. The additive noise in the network is modeled in two parts: the unstructured $N\times 1$ 
vector noise ${v}$ and the structured, or low-rank, noise ${Z}x_{\ssz}$ that lies in the subspace spanned by the columns of the $N\times L$ full-column rank matrix $Z$ (with $L\ll N$ and $x_{\ssz}$ an $L\times 1$ vector). The structured noise ${Z}{x_{\ssz}}$ can be any signal that interferes with the signal of primary interest $Wx_{\ssw}$. We assume that the columns of the matrices $W$ and $Z$ are  linearly independent so that the composite matrix $[W~Z]$ is full-column rank ($P+L\leq N$). Note that the linear independence assumption implies that the intersection between the spaces spanned by the columns of $W$ and $Z$ contains only the zero vector.

The linear data model~\eqref{eq: linear data model}, which assumes that part of the noise is a process occurring in a space of lower dimensionality, arises in {many}  signal processing applications. For example, in narrowband array processing, the measurement model is of the form~\eqref{eq: linear data model} where $Wx_{\ssw}={a}(\phi_0)x_{\ssw}$ corresponds to the signal arriving from angle $\phi_0$, and $Zx_{\ssz}=[{a}(\phi_1)\ldots{a}(\phi_M)]x_{\ssz}$ is the interference corresponding to other propagating signals from angles $\phi_1,\ldots,\phi_M$~\cite{behrens1994signal,sayed1995oblique}. The objective in this application is to enhance $Wx_{\ssw}$ and to null  $Zx_{\ssz}$. 

In {Sec.~\ref{Distributed oblique projection}}, we assume  that the model matrices $W$ and $Z$ are known, and that the objective is to estimate in a distributed manner  the useful signal $w\triangleq Wx_{\ssw}$ based on the measurement vector $y$.  We  explain how oblique projections can solve the problem by projecting onto $\cR(W)$ along the parallel direction to $\cR(Z)$, and then we show how the projection can be performed in a distributed and iterative manner.  {The concepts developed in Sec.~\ref{Distributed oblique projection} will then serve as the foundation for the design of learning algorithms in Sec.~\ref{Distributed oblique projection over mean-square-error (MSE) networks} where  the static data model~\eqref{eq: linear data model} is generalized by allowing for \emph{streaming data} scenarios, \emph{vector valued observations} at the agents, and more \emph{general observation models}.  In particular, we generalize the oblique projection framework by considering learning problems   of the form:}
 \vspace{-2mm}
\begin{equation}
\label{eq: constrained optimization problem}
\begin{split}
\argmin_{\ccy,\ccx_{\ccw},\ccx_{\ccz}}&{\left\{J^{\text{glob}}(\cy)\triangleq\sum_{k=1}^NJ_k(y_k)\right\}}\\
\st&~\cy=\cW\cx_{\ccw}+\cZ\cx_{\ccz}
\end{split}
 \vspace{-1mm}
\end{equation}
where  $J_k(y_k)$ is a {differentiable convex} cost associated with agent~$k$, $y_k$ is an $M_k\times 1$ vector, $\cy=\col\{y_k\}_{k=1}^N$,  $\cx_{\ccw}$ is a $P\times 1$ vector and $\cx_{\ccz}$ is an $L\times 1$ vector. Agent $k$ is interested in estimating $w_k$, the $k$-th subvector of $\cw=\cW\cx_{\ccw}\in\cR(\cW)$. Let $M=\sum_{k=1}^NM_k$.  The $M\times P$ and $M\times L$ matrices $\cW$ and $\cZ$ are full-column rank  ($P\ll M$ and $L\ll M$) and their columns  $\cW$ and $\cZ$ are linearly independent. {The cost $J_k(y_k)$ is assumed to be expressed as the expectation of some loss function $Q_k(\cdot)$ and written as $J_k(y_k)=\expec Q_k(y_k;\bzeta_k)$, where $\bzeta_k$ denotes the random data. We are interested in solving the problem in the stochastic setting when the distribution of the data $\bzeta_k$ is unknown. In this case, and instead of employing true gradient vectors at iteration $i$, it is common to employ approximate  vectors of the form~\cite{sayed2014adaptation}:
\vspace{-1mm}
\begin{equation}
\vspace{-1mm}
\label{eq: approximate gradient vectors}
\widehat{\nabla_{y_k}J_k}(y_k)={\nabla_{y_k}Q_k}(y_k;\bzeta_{k,i})
\end{equation}
where $\bzeta_{k,i}$ represents the data observed at iteration $i$. }

\noindent\textbf{{Notation:}}  {We use boldface letters for random quantities and normal letters for deterministic quantities. Lowercase letters denote column vectors and uppercase letters denote matrices. Unless otherwise specified, we use  calligraphic fonts  to denote block matrices and block vectors. In fact, block quantities appear  in Sec.~\ref{Distributed oblique projection over mean-square-error (MSE) networks} where inference problems over networks are considered. }

 \vspace{-2mm}
\section{Distributed oblique projection}
\label{Distributed oblique projection}
In this section, we explain how the useful signal $Wx_{\ssw}$ in~\eqref{eq: linear data model} can be estimated in a distributed and iterative manner.  To that end, we re-write  model~\eqref{eq: linear data model} as:
\begin{equation}
y=Dx+{v},
\end{equation}
where $D=[W~Z]$ and $x=\col\{x_{\ssw},x_{\ssz}\}$. By minimizing the norm of the error, namely, $\|y-Dx\|^2$, we obtain:
\begin{equation}
\label{eq: equation number 1}
x^o=(D^\top D)^{-1}D^\top y=\hspace{-1mm}\left[\begin{array}{ll}
W^\top W&W^\top Z\\
Z^\top W&Z^\top Z
\end{array}\right]^{-1}\hspace{-1mm}\left[\begin{array}{ll}
W^\top y\\
Z^\top y
\end{array}\right]
\end{equation}
Assuming, without loss of generality, that the columns of $W$ and $Z$ are orthonormal (i.e., $W^\top W=I_P$ and $Z^\top Z=I_L$), and by applying the $2\times 2$ block matrix inversion identity~\cite{kailath1980linear}, we find:
\begin{align}
w^o=Wx_{\ssw}^o&
={E}_{\ssw \ssz}y,\label{eq: subvector x 0}
\end{align}
where {${E}_{\ssw \ssz}$ is the $N\times N$ square matrix:}
\begin{equation}
\label{eq: oblique projector}
E_{\ssw \ssz}\triangleq {P}_{\ssw}(I-ZX^{-1}Z^\top P^\perp_{\ssw}),
\end{equation}
$P_{\ssw}=WW^\top$ is the orthogonal projection onto $\mathcal{R}(W)$, $ X=I-Z^\top P_WZ=Z^\top P_{\ssw}^\perp Z$, and  $ P^\perp_{\ssw}=I-{P}_{\ssw}$. The matrix $E_{\ssw \ssz}$ is referred to as the \emph{oblique projection} whose range is $\mathcal{R}(W)$ and whose null space contains $\mathcal{R}(Z)$~\cite{behrens1994signal}. It has the following properties--see~\cite[Sec. III]{behrens1994signal}:
\begin{enumerate}
\item It is equal to $W(W^\top P_{\ssz}^\perp W )^{-1}W^\top P_{\ssz}^\perp$, where $P_{\ssz}=ZZ^\top$ and $P_{\ssz}^\perp=I-P_{\ssz}$;
\item It is idempotent ($E_{\ssw \ssz}=E_{\ssw \ssz}^2$), but not symmetric;
\item Its range is $\mathcal{R}(W)$ (since $E_{\ssw \ssz}W=W$);
\item Its null space is $\mathcal{R}([Z~U])$ where $U$ spans the perpendicular space to $\mathcal{R}([W~Z])$  (${E}_{\ssw \ssz}Z=0$ and  ${E}_{\ssw \ssz}U=0$);
\item It has $P$ eigenvalues at 1 and $N-P$ eigenvalues at $0$. Its singular values are $0,1,$ or any value greater than $1$;
\item The orthogonal projection onto $\mathcal{R}({D})$ can be written as:
\begin{equation}
{P}_{\ssd}=D(D^\top D)^{-1}D^\top=E_{\ssw \ssz}+E_{ \ssz\ssw},
\end{equation}
where $E_{\ssz\ssw}=ZX^{-1}Z^\top {P}_{\ssw}^\perp$. As illustrated in {Fig.~\ref{fig: oblique projection}}, the oblique projector operator $E_{\ssw\ssz}$ projects vectors onto $\mathcal{R}(W)$ along the direction parallel to $\mathcal{R}(Z)$, and likewise for the oblique projector ${E}_{\ssz\ssw}$. Any vector $y\in\mathbb{R}^N$ can be decomposed as:
\begin{equation}
y=E_{\ssw\ssz}y+E_{\ssz\ssw}y+{P}_{\ssu}y,
\end{equation}
where ${P}_{\ssu}$ denotes the orthogonal projector onto $\mathcal{R}(U)$.
\end{enumerate}
\begin{figure}
\begin{center}
\includegraphics[scale=0.31]{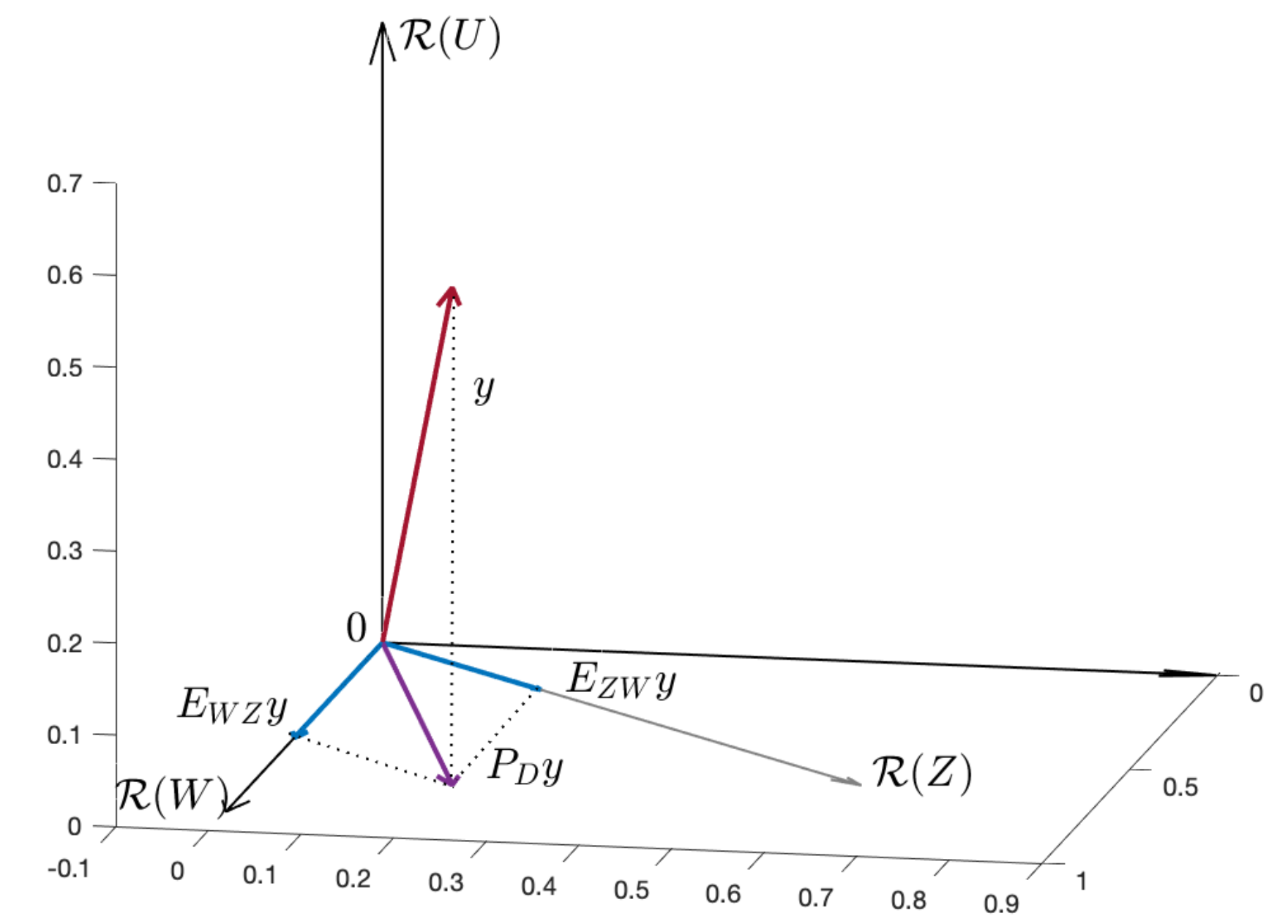}
\caption{Illustration of oblique projection.}
\label{fig: oblique projection}
\end{center}
\vspace{-2mm}
\end{figure}

The computation in~\eqref{eq: subvector x 0} is centralized since matrix $E_{\ssw\ssz}$ is dense in general, requiring nodes to send their measurements $\{y_k\}$ to a fusion center that performs the oblique projection. The objective is to compute the projection in~\eqref{eq: subvector x 0} with a network, where each node performs local computations and exchanges information only with its neighbors. We propose to replace the $N\times N$ oblique projector $E_{\ssw\ssz}$ in~\eqref{eq: subvector x 0} by an $N\times N$ matrix $A$ that satisfies the following conditions~\cite{nassif2020adaptation,dilorenzo2020distributed}:
\begin{eqnarray}
&\lim\limits_{i\rightarrow \infty} A^i=E_{\ssw\ssz}, & \label{eq: convergence condition}\\
& [A]_{k\ell}=a_{k\ell}=0,&\text{if }\ell\notin \cN_k\text{ and }k\neq\ell, \label{eq: sparsity condition}
\end{eqnarray}
where $a_{k\ell}$ denotes the $(k,\ell)$-th component of $A$. The sparsity condition~\eqref{eq: sparsity condition} characterizes the network topology and ensures local exchange of information at each iteration $i$. By replacing the projector $E_{\ssw\ssz}$ in~\eqref{eq: subvector x 0} with $A$, we obtain the following recursion at agent $k$:
\begin{equation}
\label{eq: distributed solution}
w_k(i)=\sum_{\ell\in\cN_k}a_{k\ell}w_\ell(i-1),\qquad i\geq 0
\end{equation}
where $w_k(i)$ is the estimate of $[w^o]_k$ at iteration $i$ and where node $k$ initializes its state variable with its local measurement, i.e., $w_k(-1)=y_k$. Let $w_i=\col\{w_1(i),\ldots,w_N(i)\}$. Condition~\eqref{eq: distributed solution} ensures that the network vector $w_i$ converges to the oblique projection of the initial vector $w_{-1}=y$ onto  $\mathcal{R}(W)$ along the direction parallel to $\mathcal{R}(Z)$. Necessary and sufficient conditions for~\eqref{eq: convergence condition} to hold are given in the following lemma.
\begin{lemma}
\label{lemma: oblique projection conditions}
The matrix equation~\eqref{eq: convergence condition} holds, if and only if, ${E}_{\ssw \ssz}={E}_{\ssw \ssz}^2$ (this condition is satisfied by the oblique projector ${E}_{\ssw \ssz}$) and the following conditions on $A$ are satisfied:
\begin{align}
A\,{E}_{\ssw \ssz}={E}_{\ssw \ssz},\label{eq: condition 1}\\
{E}_{\ssw \ssz}A={E}_{\ssw \ssz},\label{eq: condition 2}\\
\rho(A-{E}_{\ssw \ssz})<1,\label{eq: condition 3}
\end{align}
where $\rho(\cdot)$ denotes the spectral radius of its matrix argument. It follows that any $A$ satisfying condition~\eqref{eq: convergence condition} has one as an eigenvalue with multiplicity $P$, and all other eigenvalues are strictly less than one in magnitude.
\end{lemma}
\begin{proof}
The arguments are along the lines developed in~\cite[Appendix~A]{nassif2020adaptation} for learning under subspace constraints.
\end{proof}
If we replace ${E}_{\ssw \ssz}$ by~\eqref{eq: oblique projector} and multiply both sides of~\eqref{eq: condition 1} by $W$, we find that condition~\eqref{eq: condition 1} {implies} $A\,W=W$. Thus, the $P$ columns of $W$ are the right eigenvectors of $A$ associated with the eigenvalue $1$.  The left-eigenvectors corresponding to the eigenvalue 1 are given by $W^\top {E}_{\ssw \ssz}$. To see this, replace ${E}_{\ssw \ssz}$ by~\eqref{eq: oblique projector} and multiply both sides of~\eqref{eq: condition 2} by $W^\top$.  

Let $\wt_i\triangleq w^o-w_i$ denote the network error vector. Using~\eqref{eq: subvector x 0},~\eqref{eq: distributed solution}, and~\eqref{eq: condition 1}, we {have $w^o=Aw^o$ so that}:
\begin{align}
\wt_i&=w^o-Aw_{i-1}=(A-{E}_{\ssw \ssz})w^o-Aw_{i-1}+{E}_{\ssw \ssz}w^o.\label{eq: equation of error vector}
\end{align}
Using the fact that $w_{i-1}={A^{i}}w_{-1}={A^{i}}y$, we can write:
\begin{equation}
{{E}_{\ssw \ssz}w^o\overset{\eqref{eq: subvector x 0}}={E}_{\ssw \ssz}y\overset{\eqref{eq: condition 2}}={E}_{\ssw \ssz}A^iy={E}_{\ssw \ssz}w_{i-1}.}\label{eq: equation of error vector 1}
 \end{equation}
 Replacing~\eqref{eq: equation of error vector 1} into~\eqref{eq: equation of error vector}, we arrive at:
\begin{equation}
\wt_i=(A-{E}_{\ssw \ssz})\wt_{i-1}.\label{eq: equation of error vector 2}
\end{equation}
Condition~\eqref{eq: condition 3} guarantees convergence of the network error vector toward ${0}$, namely, $\lim_{i\rightarrow\infty}\wt_i=0$.

\smallskip
\noindent\textbf{Remark 1:}  The signal subspace model considered in~\cite{barbarossa2009distributed} can be recast in the form~\eqref{eq: linear data model} with  $W=0$ (since the noise process in~\cite{barbarossa2009distributed} is assumed to be   full-rank). In this case, the oblique projector ${E}_{\ssw \ssz}$ in~\eqref{eq: subvector x 0} reduces to the orthogonal projector ${P}_{\ssw}$ and  the convergence conditions of the distributed signal subspace projection algorithm proposed in~\cite{barbarossa2009distributed} can be obtained from conditions~\eqref{eq: condition 1}--\eqref{eq: condition 3} by replacing ${E}_{\ssw\ssz}$ by ${P}_{\ssw}$.

Appendix~\ref{app: Finding the matrix A} shows how, whenever the sparsity constraint~\eqref{eq: sparsity condition} and the signal subspace lead to a feasible problem, the matrix $A$ maximizing the convergence speed of~\eqref{eq: distributed solution}   to $w^o$ and satisfying conditions~\eqref{eq: condition 1}--\eqref{eq: condition 3} and~\eqref{eq: sparsity condition} can be found by solving an appropriate semi-definite program.

\section{{Adaptation and learning over networks in the presence of oblique projections}}
\label{Distributed oblique projection over mean-square-error (MSE) networks}
{We now consider inference problems over networks of the form~\eqref{eq: constrained optimization problem} where each agent $k$ is interested in estimating $w^o_k$, the $k$-th subvector of $\cw^o=\cW\cx^o_{\ccw}$. We first derive the centralized solution, and then we propose a distributed solution. Recall that the composite matrix $\cD\triangleq[\cW~\cZ]$ is full-column rank $(P+L\leq M)$. We also assume that the columns of $\cW$ and $\cZ$ are orthonormal. }

\vspace{-3mm}
\subsection{Centralized adaptive solution}
\label{app: Centralized solution derivation}
To  solve the constrained problem~\eqref{eq: constrained optimization problem}, we employ a penalty method and solve instead the following unconstrained problem:
\begin{equation}
\label{eq: global optimization problem 3}
\argmin_{\cy,\ccx_{\ccw},\ccx_{\ccz}}J^{\text{glob}}(\cy)+\frac{\eta}{2}\|\cy-\cW\cx_{\ccw}-\cZ\cx_{\ccz}\|^2
\end{equation}
where $\eta>0$ is a  finite large regularization parameter.  Since problem~\eqref{eq: global optimization problem 3} is convex in $\{\cy,\cx_{\ccw},\cx_{\ccz}\}$,  minimizing over $\cx=\col\{\cx_{\ccw},\cx_{\ccz}\}$ and $\cy$ in~\eqref{eq: global optimization problem 3} is equivalent to solving:
\begin{equation}
\label{eq: global optimization problem 4}
\min\limits_{\ccy} \min\limits_{\ccx_{\ccw},\ccx_{\ccz}}F(\cy,\ccx_{\ccw},\ccx_{\ccz})\triangleq J^{\text{glob}}(\cy)+\frac{\eta}{2}\|\cy-\cW\cx_{\ccw}-\cZ\cx_{\ccz}\|^2.
\end{equation}
By minimizing $F(\cy,\ccx_{\ccw},\ccx_{\ccz})$ over $\ccx_{\ccz}$, we obtain:
\begin{equation}
\label{eq: solution xz}
\ccx_{\ccz}^o=(\cZ^\top\cZ)^{-1}\cZ^\top(\cy-\cW\cx_{\ccw}).
\end{equation}
By substituting~\eqref{eq: solution xz} into~\eqref{eq: global optimization problem 4}, we arrive {at}:
\begin{equation}
\label{eq: global optimization problem 5}
\min\limits_{\ccy} \min\limits_{\ccx_{\ccw}}F'(\cy,\ccx_{\ccw})\triangleq J^{\text{glob}}(\cy)+\frac{\eta}{2}\|\cP^\perp_{{\ccz}}(\cy-\cW\cx_{\ccw})\|^2,
\end{equation}
where $\cP^\perp_{{\ccz}}=I-\cZ(\cZ^{\top}\cZ)^{-1}\cZ^{\top}$. By minimizing  $F'(\cy,\ccx_{\ccw})$ over $\ccx_{\ccw}$, we obtain:
\begin{equation}
\label{eq: solution xw}
\ccx_{\ccw}^o=(\cW^\top \cP_{\ccz}^\perp\cW)^{-1}\cW^\top \cP_{\ccz}^\perp\cy.
\end{equation}
By substituting~\eqref{eq: solution xw} into~\eqref{eq: global optimization problem 5}, we arrive {at}:
\begin{equation}
\label{eq: global optimization problem 6}
\min\limits_{\ccy}  J^{\text{glob}}(\cy)+\frac{\eta}{2}\|\cP^\perp_{{\ccz}}(I-\cE_{\ccw{\ccz}})\cy\|^2{,}
\end{equation}
where 
\begin{equation}
\label{eq: oblique projection first form}
\cE_{\ccw\ccz}=\cW(\cW^\top \cP_{\ccz}^\perp\cW)^{-1}\cW^\top \cP_{\ccz}^\perp.
\end{equation}
Problem~\eqref{eq: global optimization problem 6} can be solved using the {stochastic} gradient descent algorithm{--notice that approximate gradient vectors $\widehat{\nabla_{y_k}J_k}(\cdot)$ are used instead of true gradient vectors ${\nabla_{y_k}J_k}(\cdot)$}:
\begin{equation}
\label{eq: stochastic gradient algorithm}
\bcy_i=\bcy_{i-1}-\mu \,\col\left\{\widehat{\nabla_{y_k}J_k}(\by_{k,i-1})\right\}_{k=1}^N-\mu\eta \cP_{\ccz}^\perp(I-\cE_{\ccw{\ccz}})\bcy_{i-1}
\end{equation}
where $\mu$ a small step-size and where we used the fact that 
\begin{equation}
\label{eq: identity}
(I-\cE_{\ccw\ccz})P_\ccz^\perp(I-\cE_{\ccw\ccz})=\cP_\ccz^\perp(I-\cE_{\ccw\ccz}).
\end{equation}
Instead of implementing the update~\eqref{eq: stochastic gradient algorithm} in one step, we implement it in two successive steps according to:
\begin{equation}
\label{eq: stochastic gradient algorithm 1}
\begin{split}
\bpsi_{k,i}&=\by_{k,i-1}-\mu \widehat{\nabla J_k}(\by_{k,i-1})\\
\bcy_i&=\bpsi_i-\mu\eta \cP_{\ccz}^\perp(I-\cE_{\ccw{\ccz}})\bcy_{i-1}
\end{split}
\end{equation}
where $\bpsi_{k,i}$ is an intermediate estimate of $\cy_k$ at agent $k$ and iteration $i$ and $\bpsi_i=\col\{\bpsi_{k,i}\}_{k=1}^N$. The intermediate value  $\bpsi_{k,i}$  at node $k$ is generally a better estimate than $\by_{k,i-1}$. Therefore, we replace $\bcy_{i-1}$ by $\bpsi_{i}$ in the second step of~\eqref{eq: stochastic gradient algorithm 1}. This step is reminiscent of incremental-type approaches to optimization, which have been widely studied in the literature~\cite{bertsekas1997new,lopes2007incremental,rabbat2005quantized}.  
By doing so, and by setting $\eta=\mu^{-1}$, we obtain:
\begin{equation}
\label{eq: stochastic gradient algorithm 2}
\begin{split}
\bpsi_{k,i}&=\by_{k,i-1}-\mu \widehat{\nabla J_k}(\by_{k,i-1}),\\
\bcy_i&=\cE'_{\ccw{\ccz}}\bpsi_{i},
\end{split}
\end{equation}
where 
\begin{equation}
\cE'_{\ccw\ccz}\triangleq I- \cP_{\ccz}^\perp(I-\cE_{\ccw{\ccz}}).
\end{equation}
Using identity~\eqref{eq: identity}, we can show that {$\cE'_{\ccw{\ccz}}=(\cE'_{\ccw{\ccz}})^2$, $(\cE'_{\ccw{\ccz}})^\top=\cE'_{\ccw{\ccz}}$, $\cE'_{\ccw{\ccz}}\cW=\cW$, and  $\cE'_{\ccw{\ccz}}\cZ=\cZ$}. Now, using the fact that 
\begin{equation}
\cE_{{\ccz}\ccw}=\cP_{\ccz}\left(I-\cW(\cW^\top \cP_{{\ccz}}^\perp\cW)^{-1}\cW^\top P_{{\ccz}}^\perp\right)
\end{equation}
we {can further} show that $\cE'_{\ccw{\ccz}}=\cE_{\ccw{\ccz}}+\cE_{{\ccz}\ccw}=\cP_{\mathcal{D}}$ where $\mathcal{D}=[\cW~\cZ]$.  Finally, using~\eqref{eq: solution xw}, we obtain the estimate $\bcw_i$ of the vector {$\cw^o=\cW\cx^o_{\ccw}$} at iteration $i$:
\begin{equation}
\label{eq: solution xw iterative}
\bcw_i=\cW(\cW^\top \cP_{\ccz}^\perp\cW)^{-1}\cW^\top \cP_{\ccz}^\perp\bcy_i=\cE_{\ccw{\ccz}}\bcy_i.
\end{equation}
By combining~\eqref{eq: stochastic gradient algorithm 2} and~\eqref{eq: solution xw iterative}, we arrive {at Algorithm~\ref{alg: Centralized adaptive solution}. 
}
\begin{algorithm}[t]
  \vspace{-2mm}
\begin{subequations}
\label{eq: Centralized adaptive solution}
\begin{align}
\bpsi_{k,i}&=\by_{k,i-1}-\mu \widehat{\nabla J_k}(\by_{k,i-1}),\label{eq: centralized solution-first step}\\
\bcy_i&=\cP_{\ccd}\bpsi_{i}
,\label{eq: centralized solution-second step}\\
\bcw_i&
=\cE_{\ccw\ccz}\bcy_i.\label{eq: centralized solution-third step}
\end{align}
\end{subequations}
 \caption{Centralized adaptive solution  for solving~\eqref{eq: constrained optimization problem}}
 \label{alg: Centralized adaptive solution}
  \vspace{-5mm}
\end{algorithm}

 \vspace{-2mm}
\subsection{Distributed adaptive solution}
\label{subsec: Distributed adaptive solution}
Although  step~\eqref{eq: centralized solution-first step} is decentralized, the projection steps~\eqref{eq: centralized solution-second step} and~\eqref{eq: centralized solution-third step} in~\eqref{eq: Centralized adaptive solution} require a fusion center.  To handle the orthogonal projection, we follow similar arguments as in~\cite{nassif2020adaptation,nassif2020adaptation2} and replace the $M\times M$ matrix $\cP_{\ccd}$ in~\eqref{eq: centralized solution-second step}  by  an $M\times M$ matrix $\cC$ that satisfies the following conditions:
 \begin{eqnarray}
&\lim\limits_{i\rightarrow \infty}\cC^i={\cP}_{\ccd}, & \label{eq: convergence condition 3}\\
&~~~C_{k\ell}=[\cC]_{k\ell}=0,&\text{if }\ell\notin \cN_k\text{ and }k\neq\ell. \label{eq: sparsity condition 3}
\end{eqnarray}
with $C_{k\ell}$ denoting the $(k,\ell)$-th block of $\cC$ of size $M_k\times M_{\ell}$. Doing so, we obtain the following distributed adaptive solution at each agent $k$~\cite{nassif2020adaptation,nassif2020adaptation2}:
\begin{equation}
\label{eq: distributed adaptive-orthogonal}
\left\lbrace
\begin{split}
\bpsi_{k,i}&=\by_{k,i-1}-\mu\widehat{\nabla J_k}(\by_{k,i-1}),\\
\by_{k,i}&=\sum\limits_{\ell\in\cN_k}C_{k\ell}\bpsi_{\ell,i},
\end{split}
\right.
\end{equation}
where $\by_{k,i}$ is the estimate of $y^o_k=w^o_k+z^o_k$ at agent $k$ and iteration~$i$. It was shown in~\cite[Lemma~1]{nassif2020adaptation} that the matrix equation~\eqref{eq: convergence condition 3} holds, if and only if, the following three conditions are satisfied:
\begin{equation}
\label{eq: distributed  orthogonal projection conditions}
\cC \cP_{\ccd}= \cP_{\ccd},\quad \cP_{\ccd}\cC= \cP_{\ccd},\quad \rho(\cC- \cP_{\ccd})<1
\end{equation}

To handle the oblique projection, we  replace the $M\times M$ matrix $\cE_{\ccw\ccz}$ in~\eqref{eq: centralized solution-third step} by  an $M\times M$ matrix $\cA^S$, where $\cA$ satisfies the following conditions:
 \begin{eqnarray}
&\lim\limits_{i\rightarrow \infty}\cA^i={\cE}_{\ccw\ccz}, & \label{eq: convergence condition 2}\\
&A_{k\ell}=[\cA]_{k\ell}=0,&\text{if }\ell\notin \cN_k\text{ and }k\neq\ell, \label{eq: sparsity condition 2}
\end{eqnarray}
with $S$ a positive integer denoting the number of hops. Doing so and using the fact that $\cA^S\bpsi_i$ can be implemented in $S$ communication steps, step~\eqref{eq: centralized solution-third step}  can be replaced by the following multi-hop step at agent $k$:
\begin{equation}
\label{eq: distributed adaptive solution-step 3}
\left\lbrace
\begin{split}
\bw_{k,i}^{(s)}&=\sum\limits_{\ell\in\cN_k}A_{k\ell}\bw_{\ell,i}^{(s-1)},\quad s=1,\ldots,S,\\
\bw_{k,i}&=\bw_{k,i}^{(S)},
\end{split}
\right.
\end{equation}
with $\bw_{k,i}^{(0)}=\by_{k,i}$. To avoid the multi-hop step~\eqref{eq: distributed adaptive solution-step 3}, we propose to replace step~\eqref{eq: centralized solution-third step} by the following smoothing step:
\begin{equation}
\label{eq: smoothing step 3}
\bcw_i=(1-\nu)\cA\bcw_{i-1}+\nu\cA\bcy_i,
\end{equation}   
where $0<\nu\ll 1$ is a forgetting factor.  By combining~\eqref{eq: distributed adaptive-orthogonal} and~\eqref{eq: smoothing step 3}, we arrive at the distributed  Algorithm~\ref{alg: distributed adaptive solution},
which allows for significant communication savings when compared with the multi-hop implementation~\eqref{eq: distributed adaptive solution-step 3}. {By building upon the findings of~\cite{nassif2020adaptation,nassif2020adaptation2}, we analyze in the following section the performance of Alg.~\ref{alg: distributed adaptive solution}.} The first step~\eqref{eq: adaptation step} corresponds to the stochastic gradient step and results in  $\bpsi_{k,i}$, an intermediate estimate of $y^o_k$ at iteration $i$. This step is followed by the combination step~\eqref{eq: orthogonal step}, where node~$k$ combines the intermediate estimates $\{\bpsi_{\ell,i}\}$ from its neighbors using the combination blocks $C_{k\ell}$. The result of this step is  $\by_{k,i}$, an estimate of $y^o_k$ at iteration~$i$. In the third step, node~$k$ combines the intermediate estimates $\{\by_{\ell,i}\}$ and the previous estimates $\{\bw_{\ell,i-1}\}$ from its neighbors according to~\eqref{eq: smoothing step}. The result of this step is  $\bw_{k,i}$, an estimate of $w^o_k$ at iteration~$i$.
\begin{algorithm}[t]
  \vspace{-2mm}
\begin{subequations}
\label{eq: distributed adaptive solution}
\begin{align}
\bpsi_{k,i}&=\by_{k,i-1}-\mu\widehat{\nabla J_k}(\by_{k,i-1}),\label{eq: adaptation step}\\
\by_{k,i}&=\sum\limits_{\ell\in\cN_k}C_{k\ell}\bpsi_{\ell,i},\label{eq: orthogonal step}\\
\bw_{k,i}&=\sum\limits_{\ell\in\cN_k}A_{k\ell}((1-\nu)\bw_{\ell,i-1}+\nu\,\by_{\ell,i}).\label{eq: smoothing step}
\end{align}
\end{subequations}
 \caption{{Oblique diffusion algorithm}}
 \label{alg: distributed adaptive solution}
 \vspace{-2mm}
\end{algorithm}
 \vspace{-2mm}
 
\subsection{Performance results}
\label{subsec: Performance analysis}
Observe that the evolution of $\bpsi_i=\col\{\bpsi_{k,i}\}_{k=1}^N$ and $\bcy_i=\col\{\by_{k,i}\}_{k=1}^N$ in Alg.~\ref{alg: distributed adaptive solution} is independent of the evolution of $\bcw_i=\col\{\bw_{k,i}\}_{k=1}^N$.  We already know from~\cite[Theorem~1]{nassif2020adaptation} and~\cite[Appendix~H]{nassif2020adaptation2} that, {under some assumptions on the cost functions and gradient noises,} and after sufficient time, the iterates $\by_{k,i}$ generated by~\eqref{eq: adaptation step},~\eqref{eq: orthogonal step} converge to the true models $y^o_k$ in the mean and in the mean-square-error sense according to:
\begin{equation}
\limsup_{i\rightarrow\infty}\|\expec( y^o_k-\by_{k,i})\|=O(\mu),\qquad k=1,\ldots,N.
\end{equation}
\begin{equation}
\limsup_{i\rightarrow\infty}\expec\|y^o_k-\by_{k,i}\|^2=O(\mu),\qquad k=1,\ldots,N,
\end{equation}
for small enough $\mu$. To study the convergence w.r.t. $w^o_k$, we study the smoothing step~\eqref{eq: smoothing step 3}, 
which can be re-written as:
\begin{equation}
\label{eq: smoothing step dynamics}
\bcw_i=(1-\nu)^i\cA^i\bcw_0+\nu\sum_{j=0}^{i-1}(1-\nu)^j\cA^{j+1}\bcy_{i-j}
\end{equation}
After sufficient iterations, the influence of the initial condition in~\eqref{eq: smoothing step dynamics} can be ignored and we approximate $\bcw_i$ by the  geometric series:
\begin{equation}
\label{eq: smoothing step dynamics 1}
\begin{split}
\lim_{i\rightarrow\infty}\expec\bcw_i&\approx\nu\sum_{j=0}^{\infty}(1-\nu)^j\cA^{j+1}\cy_{\infty}=\nu\cA\left(\sum_{j=0}^{\infty}\left((1-\nu)\cA\right)^{j}\right)\cy_{\infty},
\end{split}
\end{equation}
where $\cy_{\infty}=\cy^o+O(\mu)$ and $\cy^o=\col\{y_{k}^o\}_{k=1}^N$. From Lemma~\ref{lemma: oblique projection conditions}, and using similar arguments as in~\cite[Appendix~C]{nassif2020adaptation}, we can re-write the $M\times M$ combination matrix $\cA$ in the following Jordan canonical decomposition form:
\begin{equation}
\label{eq: eigendecomposition of cA}
\cA=\cV_{\epsilon}\Lambda_{\epsilon}\cV_{\epsilon}^{-1}
\end{equation}
where $\cV_{\epsilon}=\left[\cW~\cV_{R,\epsilon}\right],$ $\Lambda_{\epsilon}=\diag\{I_P,\cJ_{\epsilon}\},$ and $\cV_{\epsilon}^{-1}=\col\{\cW^\top\cE_{\ccw\ccz},\cV_{L,\epsilon}^\top\}$. 
The matrix $\cJ_{\epsilon}$ consists of Jordan blocks 
with $\epsilon>0$ any small number and where the eigenvalue $\lambda$ may be complex but has magnitude less than one. Since $(1-\nu)\cA$ is stable (i.e., $\rho((1-\nu)\cA)<1$), we obtain:
\begin{equation}
\label{eq: infinite sum formula}
\sum_{j=0}^{\infty}\left((1-\nu)\cA\right)^{j}=(I-(1-\nu)\cA)^{-1}
\end{equation}
By replacing~\eqref{eq: eigendecomposition of cA} into~\eqref{eq: infinite sum formula}, we can write:
\begin{equation}
 \vspace{-1mm}
\label{eq: matrix formula}
\begin{split}
\nu\cA(I-(1-\nu)\cA)^{-1}
&=\cV_{\epsilon}\left[\begin{array}{cc}
I_P&0\\
0&\nu\cJ_{\epsilon}(I-\cJ_{\epsilon}+\nu\cJ_{\epsilon})^{-1}
\end{array}\right]\cV_{\epsilon}^{-1}
\end{split}
\end{equation}
 \vspace{-0.5mm}
For  $\nu\ll 1$, the above matrix~\eqref{eq: matrix formula} becomes approximately equal to $\cW\cW^\top\cE_{\ccw\ccz}=\cE_{\ccw\ccz}$ and, thus,  $\lim_{i\rightarrow\infty}\expec\bcw_i\approx\cE_{\ccw\ccz}\cy^o+O(\mu)$.

 \vspace{-0.5mm}

\section{Simulation results}
\label{sec: simulation results}
\begin{figure}
\vspace{-3mm}
\begin{center}
\includegraphics[scale=0.26]{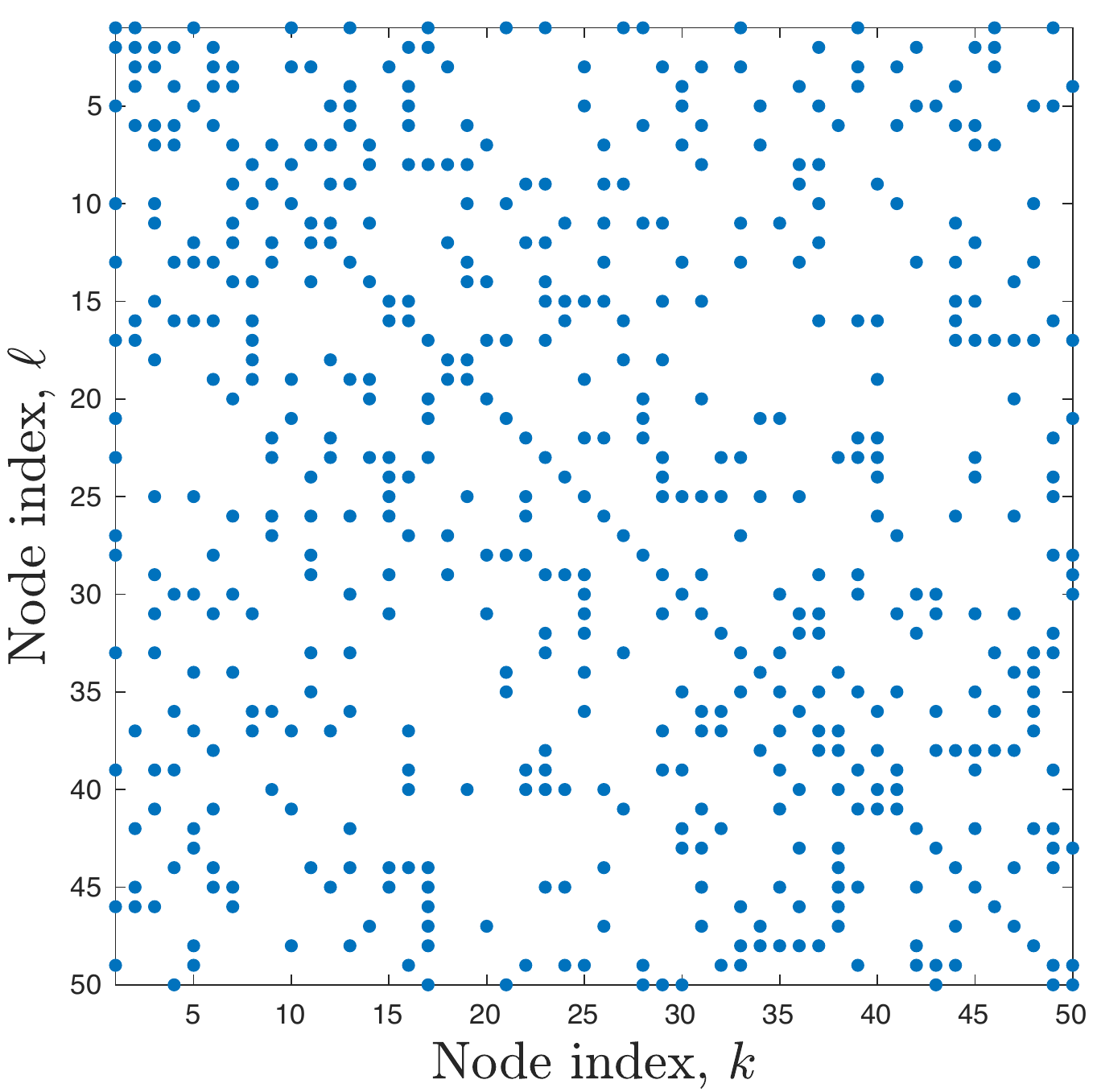}
\includegraphics[scale=0.26]{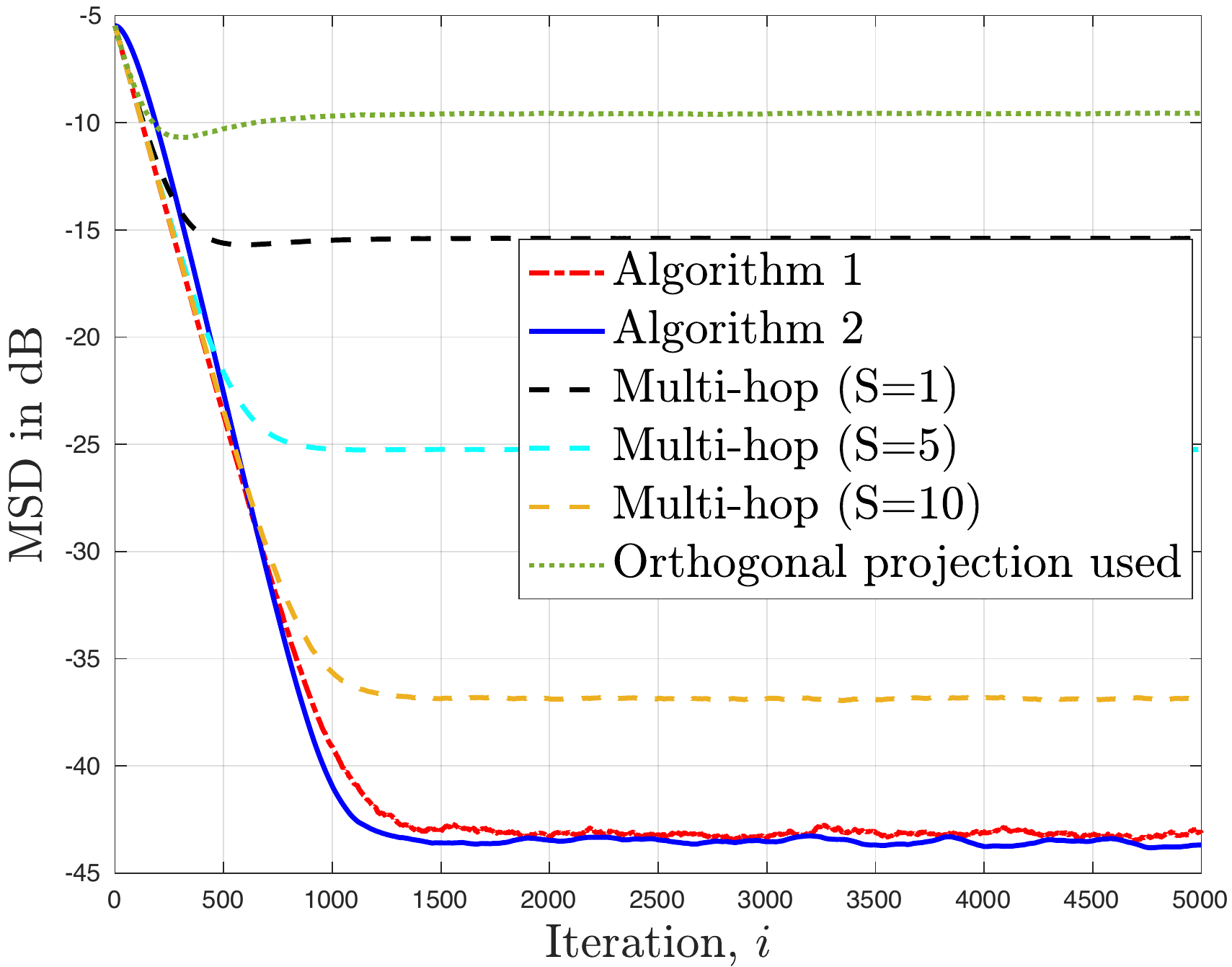}
\vspace{-3mm}
\caption{Inference in the presence of low-rank interference. \textit{(Left)} Link matrix. \textit{(Right)} Performance of Algorithms 1, 2, and the multi-hop strategy~\eqref{eq: distributed adaptive solution-step 3}.
}
\label{fig: data settings}
\end{center}
\vspace{-2mm}
\end{figure}

{{In this section, we} consider a {mean-square-error (MSE)} network {with} $N=50$ nodes and $M_k=5$ $\forall k$, generated randomly with the link matrix shown in Fig.~\ref{fig: data settings} (left). Each agent is subjected to streaming data $\{\bd_k(i),\bu_{k,i}\}$ assumed to satisfy the linear model:
\begin{equation}
\label{eq: linear model}
\bd_k(i)=\bu_{k,i}^\top (w^o_k+z^o_k)+\bv_{k}(i),
\end{equation} 
for some unknown $M_k\times 1$ vector $w^o_k$ to be estimated by agent $k$ with $\bv_{k}(i)$ denoting a zero-mean measurement noise. The vector $\cw^o=\col\{w^o_1,\ldots,w^o_N\}$  is assumed to lie in a low-dimensional subspace $\cR(\cW)$. On the other hand, $\cz^o=\col\{z^o_1,\ldots,z^o_N\}$ is assumed to lie {in a second} low-dimensional subspace $\cR(\cZ)$. The matrices $\cW$ and  $\cZ$ are  generated according to  $\cW=W\otimes I_5$ and $\cZ=z\otimes I_5$, respectively, with $W$ an $N\times 2$ randomly generated semi-orthogonal matrix ($W^\top W=I_2$) and  $z$ an $N\times 1$ randomly generated unit vector ($z^\top z=1$).  The vectors $\cw^o=\col\{w^o_k\}_{k=1}^N$ and $\cz^o=\col\{z^o_k\}_{k=1}^N$ are generated according to  $\cw^o=\cW\cx_{\ccw}^o$ and $\cz^o=\cZ\cx_{\ccz}^o$, where the vectors $\cx_{\ccw}^o$ and  $\cx_{\ccz}^o$ are randomly generated from the Gaussian distributions $\cN(0.1\times\mathds{1}_{2 M_k},I_{2 M_k})$ and $\cN(0.1\times\mathds{1}_{M_k},I_{ M_k})$, respectively. 
The processes $\{\bu_{k,i},\bv_{k}(i)\}$ are zero-mean jointly wide-sense stationary with: i) $\mathbb{E}\bu_{k,i}\bu_{\ell,i}^\top=R_{u,k}=\sigma^2_{u,k}I_5>0$ if $k=\ell$ and zero otherwise;   ii) $\mathbb{E}\bv_{k}(i)\bv_{\ell}(i)=\sigma^2_{v,k}$ if $k=\ell$ and zero otherwise; and iii) $\bu_{k,i}$ and $\bv_{\ell}(j)$ are independent {for all $k,\ell, i,j$}.  The variances $\sigma^2_{u,k}$ and $\sigma^2_{v,k}$ are generated from the uniform distributions $\text{unif}(1,4)$ and $\text{unif}(0.1,0.4)$, respectively. 
For MSE networks~\cite{sayed2013diffusion}, the risk function is of the form 
$J_k(y_k)\triangleq \frac{1}{2}\mathbb{E}|\bd_k(i)-\bu_{k,i}^\top y_k|^2$,
where $y_k=w_k+z_k$. 
Since the inference problem described in this section can be written in the form~\eqref{eq: constrained optimization problem}, 
we apply strategy~\eqref{eq: distributed adaptive solution} to solve it. 
We set $\mu=\nu=0.005$. The  matrix $\cA$ is set as the solution of problem~\eqref{eq: optimization problem-nonconvex} with $A$ and $E_{\ssw\ssz}$ replaced by $\cA$ and $\cE_{\ccw\ccz}$, respectively. We set $\epsilon=0.001$. The problem is solved via CVX package~\cite{grant2014cvx}. The matrix $\cC$ is set as the solution of~\eqref{eq: optimization problem-nonconvex} with $A$ and $E_{\ssw\ssz}$ replaced by $\cC$ and $\cP_{\ccd}$~\cite{nassif2020adaptation}. We report the network $\text{MSD}$ learning curves $\frac{1}{N}\expec\|\cw^o-\bcw_i\|^2$ in Fig.~\ref{fig: data settings} (right). The results are averaged over $200$ Monte-Carlo runs. The learning curve of the centralized solution~\eqref{eq: Centralized adaptive solution} is also reported ($\mu=0.0018$ in this  case to obtain similar convergence rate as the distributed solution). We report also the learning curves of the multi-hop strategy obtained from~\eqref{eq: distributed adaptive solution} by replacing the smoothing step~\eqref{eq: smoothing step} by the multi-hop step~\eqref{eq: distributed adaptive solution-step 3} when $S=\{1,5,10\}$. The results show that strategy~\eqref{eq: distributed adaptive solution} performs well compared with the centralized one~\eqref{eq: Centralized adaptive solution} without the need to perform $S$ communication steps at each iteration. Finally, to illustrate the importance of the oblique projection, we simulate the case where the low-rank interference problem is treated through the orthogonal projection onto $\mathcal{R}(\cW)$ (i.e., $\bw_{k,i}=\by_{k,i}$). } 
\bibliographystyle{IEEEbib}
{\balance{
\bibliography{reference}}}

\begin{thebibliography}{10}

\bibitem{barbarossa2009distributed}
S.~Barbarossa, G.~Scutari, and T.~Battisti,
\newblock ``Distributed signal subspace projection algorithms with maximum
  convergence rate for sensor networks with topological constraints,''
\newblock in {\em Proc. IEEE International Conference on Acoustics, Speech and
  Signal Processing}, Taipei, Taiwan, Apr. 2009, pp. 2893--2896.

\bibitem{sayed2014adaptation}
A.~H. Sayed,
\newblock ``Adaptation, learning, and optimization over networks,''
\newblock {\em Foundations and Trends in Machine Learning}, vol. 7, no. 4-5,
  pp. 311--801, 2014.

\bibitem{sayed2013diffusion}
A.~H. Sayed, S.~Y. Tu, J.~Chen, X.~Zhao, and Z.~J. Towfic,
\newblock ``Diffusion strategies for adaptation and learning over networks,''
\newblock {\em IEEE Signal Processing Magazine}, vol. 30, no. 3, pp. 155--171,
  2013.

\bibitem{nassif2020multitask}
R.~{Nassif}, S.~{Vlaski}, C.~{Richard}, J.~{Chen}, and A.~H. {Sayed},
\newblock ``Multitask learning over graphs: {A}n approach for distributed,
  streaming machine learning,''
\newblock {\em IEEE Signal Processing Magazine}, vol. 37, no. 3, pp. 14--25,
  2020.

\bibitem{nassif2020adaptation}
R.~{Nassif}, S.~{Vlaski}, and A.~H. {Sayed},
\newblock ``Adaptation and learning over networks under subspace
  constraints--{P}art {I}: {S}tability analysis,''
\newblock {\em IEEE Transactions on Signal Processing}, vol. 68, pp.
  1346--1360, 2020.

\bibitem{nassif2020adaptation2}
R.~{Nassif}, S.~{Vlaski}, and A.~H. {Sayed},
\newblock ``Adaptation and learning over networks under subspace
  constraints--{P}art {II}: {P}erformance analysis,''
\newblock {\em IEEE Transactions on Signal Processing}, vol. 68, pp.
  2948--2962, 2020.

\bibitem{dilorenzo2020distributed}
P.~{Di Lorenzo}, S.~{Barbarossa}, and S.~{Sardellitti},
\newblock ``Distributed signal processing and optimization based on in-network
  subspace projections,''
\newblock {\em IEEE Transactions on Signal Processing}, vol. 68, pp.
  2061--2076, 2020.

\bibitem{dimakis2010gossip}
A.~G. Dimakis, S.~Kar, J.~M.~F. Moura, M.~G. Rabbat, and A.~Scaglione,
\newblock ``Gossip algorithms for distributed signal processing,''
\newblock {\em Proceedings of the IEEE}, vol. 98, no. 11, pp. 1847--1864, Nov.
  2010.

\bibitem{nedic2009distributed}
A.~Nedic and A.~Ozdaglar,
\newblock ``Distributed subgradient methods for multi-agent optimization,''
\newblock {\em IEEE Transactions on Automatic Control}, vol. 54, no. 1, pp.
  48--61, Jan. 2009.

\bibitem{braca2008enforcing}
P.~Braca, S.~Marano, and V.~Matta,
\newblock ``Enforcing consensus while monitoring the environment in wireless
  sensor networks,''
\newblock {\em IEEE Transactions on Signal Processing}, vol. 56, no. 7, pp.
  3375--3380, Jul. 2008.

\bibitem{chouvardas2011adaptive}
S.~Chouvardas, K.~Slavakis, and S.~Theodoridis,
\newblock ``Adaptive robust distributed learning in diffusion sensor
  networks,''
\newblock {\em IEEE Transactions on Signal Processing}, vol. 59, no. 10, pp.
  4692--4707, Oct. 2011.

\bibitem{mota2015distributed}
J.~F.~C. Mota, J.~M.~F. Xavier, P.~M.~Q. Aguiar, and M.~P{\"u}schel,
\newblock ``Distributed optimization with local domains: {A}pplications in
  {MPC} and network flows,''
\newblock {\em IEEE Transactions on Automatic Control}, vol. 60, no. 7, pp.
  2004--2009, Jul. 2015.

\bibitem{koppel2016proximity}
A.~Koppel, B.~M. Sadler, and A.~Ribeiro,
\newblock ``Proximity without consensus in online multi-agent optimization,''
\newblock in {\em Proc. IEEE International Conference on Acoustics, Speech and
  Signal Processing}, Shanghai, China, May 2016, pp. 3726--3730.

\bibitem{xiao2004fast}
L.~Xiao and S.~Boyd,
\newblock ``Fast linear iterations for distributed averaging,''
\newblock {\em Systems \& Control Letters}, vol. 53, no. 1, pp. 65--78, 2004.

\bibitem{behrens1994signal}
R.~T. {Behrens} and L.~L. {Scharf},
\newblock ``Signal processing applications of oblique projection operators,''
\newblock {\em IEEE Transactions on Signal Processing}, vol. 42, no. 6, pp.
  1413--1424, 1994.

\bibitem{sayed1995oblique}
A.~H. Sayed and T.~Kailath,
\newblock ``Oblique state-space estimation algorithms,''
\newblock in {\em Proc. of 1995 American Control Conference}, Seattle, WA, USA,
  Jun. 1995, vol.~3, pp. 1969--1973.

\bibitem{kailath1980linear}
T.~Kailath,
\newblock {\em Linear systems}, vol. 156,
\newblock Prentice-Hall, Englewood Cliffs, NJ, 1980.

\bibitem{bertsekas1997new}
D.~P. Bertsekas,
\newblock ``A new class of incremental gradient methods for least squares
  problems,''
\newblock {\em SIAM J. Optim.}, vol. 7, no. 4, pp. 913--926, 1997.

\bibitem{lopes2007incremental}
C.~G. Lopes and A.~H. Sayed,
\newblock ``Incremental adaptive strategies over distributed networks,''
\newblock {\em IEEE Trans. Signal Process.}, vol. 55, no. 8, pp. 4064--4077,
  Aug. 2007.

\bibitem{rabbat2005quantized}
M.~G. Rabbat and R.~D. Nowak,
\newblock ``Quantized incremental algorithms for distributed optimization,''
\newblock {\em IEEE J. Sel. Areas Commun.}, vol. 23, no. 4, pp. 798--808, 2005.

\bibitem{grant2014cvx}
M.~Grant and S.~Boyd,
\newblock ``{CVX}: Matlab software for disciplined convex programming, version
  2.1,'' \url{http://cvxr.com/cvx}, Mar. 2014.

\end{thebibliography}
\newpage

\begin{appendices}
\section{Finding the matrix $A$}
\label{app: Finding the matrix A}
In this appendix, we are interested in finding a matrix $A$, consistent with the given graph (i.e., satisfying condition~\eqref{eq: sparsity condition}), that maximizes  the convergence speed of system~\eqref{eq: distributed solution}, while guaranteeing the convergence of~\eqref{eq: distributed solution} to the desired vector $w^o$ in~\eqref{eq: subvector x 0} (i.e., satisfying conditions~\eqref{eq: condition 1},~\eqref{eq: condition 2}, and~\eqref{eq: condition 3}). We use  the \emph{per-step} convergence factor defined by~\cite{xiao2004fast}:
\begin{align}
\mathcal{C}(A)&=\sup_{w_{i-1}\neq w^o}\frac{\|w^o-w_i\|_2}{\|w^o-w_{i-1}\|_2}\notag\\
&\hspace{-1mm}\overset{\eqref{eq: equation of error vector 2}}=\sup_{\bxt_{i-1}\neq 0}\frac{\|(A-E_{\ssw\ssz})\wt_{i-1}\|_2}{\|\wt_{i-1}\|_2}=\|A-E_{\ssw\ssz}\|_2,
\end{align}
as a measure of the speed of convergence, where $\|X\|_2$ denotes the spectral norm of the matrix $X$ (i.e., its largest singular value). Thus, we consider the following spectral norm minimization problem for finding $A$:
\begin{equation}
\label{eq: optimization problem-nonconvex}
\begin{split}
\minimize_{A}&\qquad\mathcal{C}(A)=\|A-E_{\ssw\ssz}\|_2\\
\st&\qquad A\,E_{\ssw\ssz}=E_{\ssw\ssz},~E_{\ssw\ssz}A=E_{\ssw\ssz}\\
&\qquad\rho(A-E_{\ssw\ssz})<1\\
&\qquad [A]_{k\ell}=0,~\text{if }\ell\notin\cN_k\text{ and }k\neq \ell
\end{split}
\end{equation}
Problem~\eqref{eq: optimization problem-nonconvex} is a non-convex optimization problem since the spectral radius function is non-convex, and thus the inequality constraint function {$\rho(A-E_{\ssw\ssz})<1$} is non-convex. In the following, we use the convex spectral norm instead of the spectral radius function. Particularly, we replace the constraint $\rho(A-E_{\ssw\ssz})<1$ by $\|A-E_{\ssw\ssz}\|_2\leq 1-\epsilon$. Since the spectral radius of a matrix is bounded by any of its norms, finding an $A$ such that $\|A-E_{\ssw\ssz}\|_2\leq 1-\epsilon$ (with $0<\epsilon\ll 1$) ensures that $A$ satisfies $\rho(A-E_{\ssw\ssz})<1$. The resulting problem can be expressed as the semi-definite program (SDP):
\begin{equation}
\label{eq: optimization problem- convex-SDP}
\begin{split}
\minimize_{A,s}&\qquad s\\
\st&\qquad\left[\begin{array}{cc}
sI&A-E_{\ssw\ssz}\\
(A-E_{\ssw\ssz})^\top&sI
\end{array}\right]\succeq {0}\\
&\qquad s\leq 1-\epsilon\\
&\qquad A\,E_{\ssw\ssz}=E_{\ssw\ssz},~E_{\ssw\ssz}A=E_{\ssw\ssz}\\
&\qquad [A]_{k\ell}=0,~\text{if }\ell\notin\cN_k\text{ and }k\neq \ell
\end{split}
\end{equation}
where $\succeq$ denotes matrix inequality. This follows from the fact that the $2\times 2$ block matrix is positive semi-definite if 
\begin{itemize}
\item Its diagonal entry $sI$ is positive definite (i.e., $s>0$);
\item Its Schur complement $sI-(A-E_{\ssw\ssz})^\top(A-E_{\ssw\ssz})/s$ is positive semi-definite (i.e., $s^2I\succeq(A-E_{\ssw\ssz})^\top(A-E_{\ssw\ssz})$).
\end{itemize}
Thus, minimizing $s$ is equivalent to minimizing $\|A-E_{\ssw\ssz}\|_2$ since $\|X\|_2=(\lambda_{\max}(X^\top X))^{\frac{1}{2}}$. The SDP~\eqref{eq: optimization problem- convex-SDP} can be solved efficiently using convex optimization packages such as~\cite{grant2014cvx}.

It should be noted that not all network topologies satisfying~\eqref{eq: sparsity condition} guarantee the {existance} of an $A$ satisfying condition~\eqref{eq: convergence condition}. However, in the distributed processing under subspace constraints framework~\cite{barbarossa2009distributed,nassif2020adaptation,dilorenzo2020distributed}, it is common to assume that the sparsity constraint~\eqref{eq: sparsity condition} and the signal subspace lead to a feasible problem.
\end{appendices}
\end{document}